\documentclass[letterpaper, 10 pt, conference]{ieeeconf}

\IEEEoverridecommandlockouts      
\usepackage{epsfig}
\usepackage{epstopdf}
\usepackage{cite}
\usepackage{graphicx}
\graphicspath{{Figures/}}
\usepackage{url}
\usepackage{stfloats}
\usepackage{amsmath}
\usepackage{amsfonts}   
\usepackage{amsopn}     
\usepackage{psfrag}
\usepackage[tight]{subfigure}
\usepackage{mdwtab}     
\usepackage{enumerate}
\usepackage{color}
\usepackage{accents}

\newtheorem{theorem}{\bf Theorem}

\newtheorem{lemma}{\bf Lemma}

\newtheorem{definition}{\bf Definition}
\newtheorem{rem}{\bf Remark}

\begin{document}

\title{\LARGE \bf Safe Adaptive Switching among Dynamical Movement Primitives: Application to 3D Limit-Cycle Walkers}

\author{Sushant~Veer and Ioannis~Poulakakis
\thanks{S. Veer and I. Poulakakis are with the Department of Mechanical Engineering, University of Delaware, Newark, DE 19716, USA;  e-mail: {\tt\small \{veer, poulakas\}@udel.edu.}}
\thanks{This work is supported in part by NSF CAREER Award IIS-1350721 and by NRI-1327614.}
}

\maketitle

\begin{abstract}

Complex motions for robots are frequently generated by switching among a collection of individual movement primitives. We use this approach to formulate robot motion plans as sequences of primitives to be executed one after the other. When dealing with \emph{dynamical} movement primitives, besides accomplishing the high-level objective, planners must also reason about the effect of the plan's execution on the safety of the platform. This task becomes more daunting in the presence of disturbances, such as external forces. To alleviate this issue, we present a framework that builds on rigorous control-theoretic tools to generate safely-executable motion plans for externally excited robotic systems. Our framework is illustrated on a 3D limit-cycle gait bipedal robot that adapts its walking pattern to persistent external forcing.
\end{abstract}

\IEEEpeerreviewmaketitle

\section{Introduction}
\label{sec:intro}

Robots operating in the real world are expected to encounter a wide range of exogenous input signals due to contact or other types of interaction with a possibly time-varying, stochastic environment. Depending on the task, external signals may represent commands that need to be followed or disturbances that must be attenuated. A diverse collection of suitable primitive motions, and the capability to \emph{switch} among them, can provide a sufficiently rich repertoire of behaviors for adapting to or compensating for such signals. For example, in a human-robot physical collaboration scenario~\cite{motahar2015impedance, motahar2017steering}, a force exerted by a human with the intention of decelerating the robot can be followed by switching to a primitive of lower velocity. On the other hand, an undesirable force that pushes against a robot, which is tasked to maintain a constant desired velocity, can be compensated by switching to a primitive of higher velocity.

Adopting a dynamical systems perspective, one approach to characterizing primitive motions is to represent them as attractors of dynamical systems, termed dynamical movement primitives\footnote{The terms dynamical movement primitives and dynamical motion primitives are interchangeably used in the paper.} (DMPs) in the relevant literature~\cite{hogan2013dynamic,ijspeert2013dynamical,burridge1999sequential}. Adjusting the ``landscape'' of such attractors through coupling terms can realize both discrete and rhythmic motion patterns of high complexity, allowing a robot to perform challenging tasks in its workspace~\cite{ijspeert2013dynamical}. An alternative way to generate sufficiently rich robot behaviors is to consider a discrete collection of suitably selected DMPs, and allow switching among them~\cite{burridge1999sequential,tedrake2010}. 
This paper focuses on the latter, and it leverages recent theoretical developments in~\cite{veer2018ultimate, veer2017poincare} to provide \emph{explicitly computable} sufficient conditions that provably guarantee the safety of the robotic platform as it switches among different DMPs to accomplish a task and compensate for disturbances.

Safe switching among DMPs occupies a significant body of work in the literature of robotic motion planning. Inspired by early work in~\cite{mason1985mechanics}, a tree of positively-invariant Lyapunov funnels was constructed in \cite{burridge1999sequential} to generate controller switching policies that drive a robot to a goal while ensuring safety of the robotic platform. Later work in~\cite{tedrake2010} provided a computationally tractable approach to estimate the Lyapunov funnels construction of~\cite{burridge1999sequential} based on Sum-of-Squares (SoS) programming \cite{parrilo2000structured}. This method and its extensions have been experimentally successful for planning motions of a wheeled robot~\cite{conner2011integrating}, ball-bot~\cite{nagarajan2013integrated}, and a fixed-wing airplane~\cite{majumdar2017funnel}. Other methods address composability of dynamic primitives, capturing the associated constraints in the form of maneuver~\cite{frazzoli2005maneuver} or timed~\cite{bouyer2017timed} automata. None of the aforementioned methods deals with persistent disturbances  except~\cite{majumdar2017funnel}, which though requires knowledge of the disturbed dynamics to ensure safe operation; see \cite[Section~4.3.1]{majumdar2017funnel}. On the contrary, the conditions provided here for safe switching in the presence of persistent disturbances rely only on the zero-disturbance stability properties of the individual DMPs, and thus are \emph{agnostic} to the disturbances.

Our focus in this paper is on rhythmic motions of dynamically-stable robots; typical examples of such systems include walking or running machines~\cite{raibert1985legged} and flapping-wing flying robots~\cite{ramezani2017describing}. Mathematically, such motions can be idealized as attracting limit cycles that capture the fundamental oscillatory behavior of the underlying energy transformations. Restricting attention to legged robots, a variety of methods have been proposed to stabilize limit-cycle locomotion behaviors, including  hybrid zero dynamics (HZD)~\cite{westervelt2007feedback}, geometric control reduction~\cite{Gregg2010Geometric}, and virtual holonomic constraints~\cite{freidovich2009passive}, just to name a few. These methods aim at stabilizing \emph{individual} limit cycles that result in locomotion behaviors with certain desired attributes. However, to address the challenges faced by a legged robot moving in a time-varying environment, these individual limit cycles must be composed in response to external stimuli to form more complex motion patterns. 
Enabling such compositions is at the core of the proposed framework, the practical value of which lies on its ability to take such off-the-shelf robot controllers for individual primitive behaviors and switch among them in response to external signals, while affording rigorous safety guarantees.
It should be noted that switching among limit-cycle gaits has been explored in the context of various applications, including navigation in environments cluttered by obstacles~\cite{gregg-planning2012, motahar2016composing, veer2017driftless}, speed adaptation~\cite{veer2017continuum, bhounsule2018switching}, and robustness to disturbances~\cite{saglam2013switching, quan2016dynamic, da2016first}.
With the exception of~\cite{saglam2013switching} which provides stochastic stability guarantees, the switching policies discussed in the aforementioned methods do not account for persistent external excitation.

In this paper, building upon our recent work \cite{veer2018ultimate}, we present a general framework that formulates switching among externally excited DMPs as a switched system with multiple equilibria under disturbances. However, in \cite{veer2018ultimate}, we assumed global stability properties for the individual systems and demanded knowledge of the external signal's effect on them---either of which rarely hold true for robotic systems. The method we provide here, remedies these limitations by allowing us to furnish safety guarantees for switching under disturbances by studying an \emph{unperturbed} switched system (Theorem~\ref{thm:dist-agnostic}). The framework is demonstrated for locomotion adaptation of a 3D bipedal walker that collaborates with a leader by switching among limit-cycle primitives. 


\begin{figure}[b]
\vspace{-0.18in}
\begin{centering}
\includegraphics[width=0.95\columnwidth]{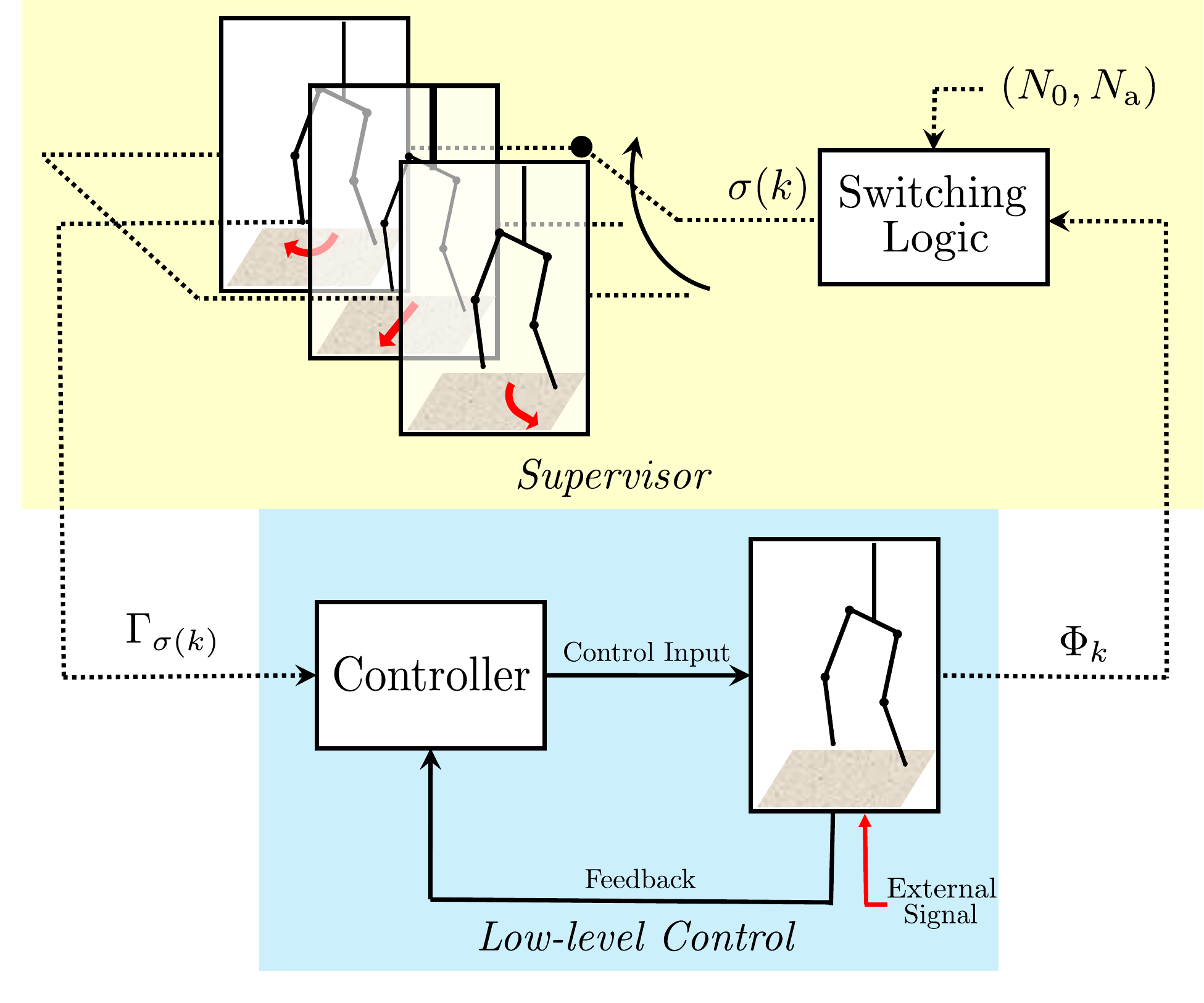} 
\par\end{centering}
\vspace{-0.2in}
\caption{Adaptive planning block diagram. The high-level supervisor is shown in yellow and the low-level control loop is shown in blue.}
\label{fig:block-diagram} 
\end{figure}

\section{Overview: Switching Framework}
\label{sec:overview}

In this section we provide an overview of our motion planning framework; technical details are relegated for later sections. Our planning framework is organized in two hierarchical levels as shown in Fig.~\ref{fig:block-diagram}. On the high level, there is a supervisor that comprises a library of primitives and a switching logic that governs the choice of the primitive to be implemented on the robot. On the low level, there is a feedback control loop which executes the controller of the primitive chosen by the supervisor. It is the supervisor's role to orchestrate switching among members of the primitive library in accordance with a prescribed high-level objective---such as navigate to a goal while avoiding obstacles \cite{motahar2016composing} or adapt to a collaborating leader \cite{veer2017supervisory}. However, due to the dynamics of the system and possible external excitation, the generated sequence of primitive switches may lead to instability, requiring the switching logic to reason about the dynamic limitations of the robot; to address these limitations, we consider the effect of switching among primitives on the robot dynamics.

We work with primitives that take the form of dynamical systems with equilibrium behaviors---equilibrium points or limit cycles. Hence, executing a motion plan naturally induces a switched system structure to the dynamics of the robot. In Section~\ref{sec:switch-guarantees}, we analyze the switched system that arises and identify a class of switching signals that can be safely executed by the robot despite disturbances. This class of switching signals is characterized by a lower bound on the average time-gap between any two consecutive switches---formally known as average dwell-time. Informing the supervisor about this constraint allows the switching logic to generate motion plans that can be safely executed. 

This framework can be used for adaptive planning in various scenarios, such as navigation of a robot in the presence of disturbances \cite{majumdar2017funnel}, adapting to rough terrain \cite{quan2016dynamic}, and teleoperation \cite{yang2018online}, to name a few. In this paper, we apply the framework to a task that involves physical cooperation between a bipedal robot and a leading collaborator---a human or a robot. The success of such a task hinges on the ability of the biped to adapt its walking pattern in response to the leader's intended trajectory, which is not explicitly available but is indirectly known through an interaction force applied by the leader. Hence, the interaction force serves as a command signal that the biped must adapt to. To achieve this, we supply the supervisor with a library of limit-cycle primitives---straight walking, turn right, and turn left---as shown in Fig.~\ref{fig:block-diagram}. The supervisor is provided with an appropriate input $\Phi_k$ that encodes the leader's intention through the interaction force, and an average dwell-time constraint, which it uses to \emph{safely} adapt the biped's motion to the leader's intended trajectory. Section~\ref{sec:example} of the paper will discuss this scenario in detail.



\section{Safe Switching under Disturbances}
\label{sec:switch-guarantees}

This section formalizes the concepts that underlie the framework briefly outlined above; see Fig~\ref{fig:block-diagram}. 


\subsection{Library of Motion Primitives}
\label{subsec:lib-prim}
 

In the proposed approach, motion primitives are characterized by point attractors of continuous- or discrete-time dynamical systems together with the vector fields capturing the corresponding dynamic behavior. Limit cycle primitives---such as those frequently employed in rhythmic behaviors---can be incorporated in our approach naturally, as they can be associated with point attractors of suitably constructed discrete-time systems via the method of Poincar\'e~\cite{GUHO96}. Hence, in what follows we will develop the method in the context of discrete-time systems, noting that analogous results hold for continuous-time systems as well~\cite{veer2018ultimate}. 

Let $\mathcal{P}$ be a finite index set and consider a collection of discrete-time nonlinear systems 
\begin{align}\label{eq:discrete_system}
x_{k+1} = f_p(x_k) \enspace, & & p\in\mathcal{P} \enspace,
\end{align}
where $x$ represents the state vector evolving in a space $\mathcal{X}_p\subseteq \mathbb{R}^n$, $f_p$ is a vector field describing the dynamics of the system and $k$ denotes discrete time. We can now formalize our notion of motion primitives, which are defined as two-tuples consisting of a vector field $f_p$  and a corresponding equilibrium (fixed) point $x_p^*$ satisfying $x_p^*=f_p(x_p^*)$; that is,
\begin{align}\label{eq:primitives}
\mathcal{G}_p:=\{f_p,x_p^*\} \enspace, & & p\in\mathcal{P} \enspace.
\end{align}
The totality $\mathbb{G} := \{ \mathcal{G}_p ~|~ p \in \mathcal{P} \}$ of the motion primitives \eqref{eq:primitives} defines a library of motion primitives.  

To ensure safe operation, $\mathbb{G}$ will only include motion primitives that correspond to exponentially stable fixed points. Safety certificates of this form can be obtained through the notion of a  Lyapunov function~\cite{khalil2002nonlinear}; i.e., a positive definite, radially unbounded, decrescent function  $V_p$ that satisfies  
\begin{align}
V_p(f_p(x)) & \leq \lambda V_p(x) \label{eq:V-2-0-inp}\enspace,
\end{align}
where\footnote{For simplicity, we assume that the convergence rate $\lambda$ is the same for all primitives; if this is not the case, we can always choose $\lambda:=\max_{p\in\mathcal{P}} \lambda_p$.} $0<\lambda<1$. 
To implement switching policies among motion primitives that afford rigorous safety guarantees, we will also need an estimate of the basin-of-attraction (BoA) $\mathcal{B}_p$ associated with each $\mathcal{G}_p$ in $\mathbb{G}$. To conveniently characterize such estimates, we use sublevel sets of Lyapunov functions verified through sums-of-squares (SoS) programming~\cite{parrilo2000structured}. The dimensional reduction afforded by the controllers employed in Section~\ref{sec:example} greatly improves computational efficiency in obtaining such estimates.     

\begin{rem}
The motion primitives defined by \eqref{eq:primitives} can be employed directly to plan or adapt rhythmic behaviors to external signals; a concrete example is provided in Section~\ref{sec:example}. Indeed, such tasks can be facilitated via switching among limit cycle motion primitives $\{ \mathcal{O}_p ~|~ p \in \mathcal{P} \}$. By the method of Poincar\'e, however, each limit cycle $\mathcal{O}_p$ can be naturally associated with a  fixed point $x_p^*$ of a discrete-time system \eqref{eq:discrete_system} with $f_p$ being the corresponding Poincar\'e map~\cite{GUHO96}. 
\end{rem}

\subsection{Conditions for Safe Switching}
\label{subsec:switch-system}

A motion planner---or, a supervisor---is responsible for monitoring the state of the system as it interacts with its environment and deciding which motion primitive $\mathcal{G}_p$ out of the library $\mathbb{G}$ must be implemented at each time instant. This decision can be represented as a ``descending'' switching signal $\sigma:\mathbb{Z}_+\to\mathcal{P}$, which maps the current time $k$ to the index $p=\sigma(k)\in\mathcal{P}$ of the member $\mathcal{G}_{\sigma(k)}$ of $\mathbb{G}$ that must be executed at $k$. The process gives rise to a discrete switched system with multiple equilibria that has the form
\begin{equation}\label{eq:switched-system}
x_{k+1} = f_{\sigma(k)}(x_k) \enspace.
\end{equation}

Quantifying safety for systems, the behavior of which is governed by a switched system like \eqref{eq:switched-system}, can be challenging; yet, such systems emerge in a wide range of applications where a supervisor chooses among different controllers, each being suited for a particular mode of operation. Switching in \eqref{eq:switched-system} effectively causes the system to ``shift'' to a different point attractor, and thus persistent switching in response to constantly varying environmental or task conditions causes the system to be in a ``permanent'' transient phase, \emph{never} converging to any of the underlying equilibrium states. The resulting evolution can be highly irregular.  

In this work, we will adopt the following set-characterization for safety. The switched system \eqref{eq:switched-system} will be considered safe, if the following condition holds 
\begin{align}\label{eq:safety_criterion}
\big\{ x^*_p \in \bigcap_{p\in\mathcal{P}} \mathcal{B}_p \text{~~and~~}
x_k \in \bigcap_{p\in\mathcal{P}} \mathcal{B}_p \text{~for all~} k \in \mathbb{Z}_+ \big\}\enspace,
\end{align}
where $\mathcal{B}_p$ is an estimate of the BoA associated with the primitive $\mathcal{G}_p$. In words, this condition implies that the state of \eqref{eq:safety_criterion} always remain trapped in a compact subset of the state space, which is explicitly characterized as the intersection of the estimates of the BoAs of all the motion primitives and includes all the equilibrium points. Our motivation for adopting this criterion is twofold. First, it implies that the state remains bounded. Second, it guarantees that if switching were to cease, the system would return to the equilibrium corresponding to the most recently implemented primitive.

The question we address next is to identify a class of switching signals $\sigma$ which guarantee that the condition \eqref{eq:safety_criterion} is satisfied. Loosely speaking, we will require that switching is sufficiently slow on average. Intuitively, we require that the average time gap between any two consecutive switches is sufficiently long to ensure that the energy dissipated due to the exponentially stable nature of each primitive $\mathcal{G}_p$ dominates the possible energy gain due to a switch, resulting in an overall energy reduction. To make this precise, we will use the notion of the average dwell time, introduced in~\cite{hespanha1999stability}.
\begin{definition}\label{def:avg-dwell}
A switching signal $\sigma:\mathbb{Z}_+\to\mathcal{P}$ has average dwell-time $N_{\rm a}>0$ if the number of switches $N_\sigma(k,\underline{k})$ in any discrete interval $[\underline{k},k)\cap\mathbb{Z}_+$ satisfies
\begin{align}\label{eq:avg-dwell-time-def}
N_\sigma(k,\underline{k}) \leq N_0 + \frac{k-\underline{k}}{N_{\rm a}} \enspace, & & \forall k\geq \underline{k} \geq 0 \enspace,
\end{align}
where $k,\underline{k}\in\mathbb{Z}_+$ and $N_0>0$ is a finite constant.
\end{definition}

Below, we introduce a set construction that will allow us to obtain explicitly computable expressions for an average dwell time constraint, which guarantees that  \eqref{eq:safety_criterion} is satisfied; a rigorous justification can be found in \cite{veer2018ultimate}. 
We begin by choosing a $\kappa>0$ and defining the $\kappa$-sublevel set of $V_p$
\begin{equation}\nonumber
\mathcal{M}_p(\kappa):=\{x\in\mathbb{R}^n~|~V_p(x)\leq \kappa\} \enspace;
\end{equation}
see Fig.~\ref{fig:set-construction}. Let $\mathcal{M}(\kappa):=\bigcup_{p\in\mathcal{P}} \mathcal{M}_p(\kappa)$ denote the union of these subsets over $\mathcal{P}$, and define
\begin{equation}\label{eq:omega-def}
\omega(\kappa):=\max_{p\in\mathcal{P}}\max_{x\in\mathcal{M}(\kappa)} V_p(x) \enspace.
\end{equation}
Then, it can be seen that $ \mathcal{M}(\kappa) \subseteq \bigcap_{p\in\mathcal{P}} \mathcal{M}_p(\omega(\kappa))$ as in Fig.~\ref{fig:set-construction}. Effectively, $\omega$ ``inflates'' the sets $\mathcal{M}_p(\kappa)$ to the sets $\mathcal{M}_p(\omega(\kappa))$, the intersection of which contains $\mathcal{M}(\kappa)$.  

To bound possible ``energy" gain due to switching, let\footnote{Notation: If $S$ is a set, then $\accentset{\circ}{S}$ denotes its interior.}
\begin{equation}\label{eq:mu-def}
\mu(\kappa):=\max_{p,q\in\mathcal{P}} \max_{x\in \mathcal{B}_p\setminus \accentset{\circ}{\mathcal{M}}_p(\kappa)} \frac{V_q(x)}{V_p(x)} \enspace,
\end{equation}
which captures the ratio of all Lyapunov functions and is well-defined since $\mathcal{P}$ is finite and $\mathcal{B}_p\setminus \accentset{\circ}{\mathcal{M}}_p(\kappa)$ is compact; the exclusion of $\accentset{\circ}{\mathcal{M}}_p(\kappa)$ is to prevent the denominator of \eqref{eq:mu-def} from approaching 0, as $V_p(x_p^*)=0$.
Note also that the interchangeability of $p$ and $q$ implies that $\mu(\kappa)\geq 1$.

With $\omega(\kappa)$ and $\mu(\kappa)$ as in \eqref{eq:omega-def} and \eqref{eq:mu-def} and $\lambda$ as in \eqref{eq:V-2-0-inp}, a lower bound on the average dwell-time can be computed as
\begin{equation}\label{eq:avg-dwell}
\overline{N}_{\rm a}= \frac{\ln{\mu(\kappa)}}{\ln{(\epsilon/\lambda)}} \enspace,
\end{equation}
where $\epsilon$ is an arbitrary constant in $(\lambda,1)$. Switching signals with $N_\mathrm{a} \geq \overline{N}_{\rm a}$ guarantee that the evolution of \eqref{eq:switched-system} remains bounded. 
Intuitively, a large $\mu$ implies that the ``energy'' gain due to switching may be large, causing $\overline{N}_{\rm a}$ to increase, so that choosing $N_{\rm a} \geq \overline{N}_{\rm a}$ in \eqref{eq:avg-dwell-time-def} will result to slower switching on average. Similarly, a slow rate of convergence $\lambda$ during the interval between switches will have a similar effect. 

Satisfying $N_\mathrm{a} \geq \overline{N}_{\rm a}$ ensures boundedness of the state, however, according to \eqref{eq:safety_criterion} we also need to ensure that the corresponding bounded ``trapping'' set lies within $\bigcap_{p\in\mathcal{P}}\mathcal{B}_p$. It turns out that, to comply with this requirement, the chosen $\kappa$ must be such that the condition 
\begin{equation}\label{eq:feasible-0-inp}
\overline{\mathcal{M}}:=
\mathcal{M}(\mu(\kappa)^{\overline{N}_0}\omega(\kappa)) \subset \bigcap_{p\in\mathcal{P}} \accentset{\circ}{\mathcal{B}}_p \enspace.
\end{equation}
is verified for some $\overline{N}_0\geq 1$. Then, for any $\sigma$ with 
\begin{equation}
N_0 \leq \overline{N}_0 \text{~~and~~} N_{\rm a}\geq \overline{N}_{\rm a} \enspace,
\end{equation}
the solution of \eqref{eq:switched-system} starting from any $x_0 \in \bigcap_{p\in\mathcal{P}} \mathcal{M}_p(\omega(\kappa))$ will remain within $\bigcap_{p\in\mathcal{P}}\mathcal{B}_p$ for all $k\in\mathbb{Z}_+$. Note also that since $\mathcal{M}$ represents the union of sublevel sets, it includes all the equilibrium points, thereby \eqref{eq:feasible-0-inp} implies that the first part of the condition \eqref{eq:safety_criterion} is also satisfied.

\begin{figure}[t]
\vspace{+0.02in}
\begin{centering}
\includegraphics[width=0.65\columnwidth]{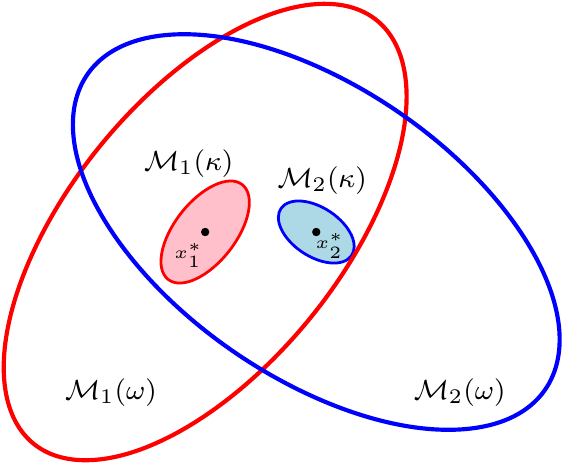} 
\par\end{centering}
\vspace{-0.1in}
\caption{Illustration of the set construction. The sublevel sets for system 1 are in red and the sublevel sets of system 2 are in blue.}
\vspace{-0.2in}
\label{fig:set-construction} 
\end{figure}

\begin{rem}
The rigorous justification of this result---which is of ``local'' nature---is due to Theorem~\ref{thm:dist-agnostic} below, which in fact establishes that the bounds $(\overline{N}_0, \overline{N}_{\rm a})$ remain valid even in the case where the system is perturbed by external disturbances, as long as they are sufficiently small. The proof of this theorem, which is made available in the appendix, also provides a clear motivation for the set constructions described above.         
For the purpose of implementation though, the procedure of obtaining $(\overline{N}_0, \overline{N}_{\rm a})$ can be summarized as follows. First, select a $\kappa>0$ and compute $\omega(\kappa)$ and $\mu(\kappa)$ by \eqref{eq:omega-def} and \eqref{eq:mu-def}, respectively. If \eqref{eq:feasible-0-inp} holds for some $\overline{N}_0 \geq 1$, then safe switching is achieved by providing the supervisor with the numbers $(\overline{N}_0, \overline{N}_{\rm a})$ where $\overline{N}_{\rm a}$ is obtained by \eqref{eq:avg-dwell}. Otherwise, choose a new $\kappa$ and repeat the procedure. 
\end{rem}
\begin{rem}
Note that, for a given $\kappa>0$, obtaining $\omega(\kappa)$ and $\mu(\kappa)$ numerically may be computationally challenging, particularly for high-dimensional systems. However, for quadratic Lyapunov functions---as is frequently the case in practical applications---an upper bound for $\omega$ and $\mu$ can be \emph{analytically} computed using \cite[Proposition~1]{veer2018ultimate}.
\end{rem}
 



\subsection{Conditions for Safe Switching: The case of disturbances}
\label{subsec:switch-system-disturb}

In this section, we state a theorem that rigorously justifies the assertions informally made in the previous section. With a slight abuse of notation, consider the perturbed version 
\begin{align}\label{eq:switch-system-disturb}
x_{k+1} = f_{\sigma(k)}(x_k,d_k) \enspace,
\end{align}
of the switched system \eqref{eq:switched-system}. In \eqref{eq:switch-system-disturb} the disturbance signal is represented as a sequence $d:=\{d_k\}_{k\in\mathbb{Z}_+}$ with values in the set of admissible disturbances\footnote{The disturbance $d_k$ can be a continuous function of time rendering $\mathcal{D}$ the mathematical structure of a Banach space and $f_p$ a nonlinear functional \cite{royden2010real}; indeed this will be the case in Section~\ref{sec:example}. } $\mathcal{D}$. Let $\|\cdot\|_{\mathcal{D}}$ be the norm on $\mathcal{D}$ and define $\|d\|_\infty:=\sup_{k\in\mathbb{Z}_+} \|d_k\|_{\mathcal{D}}$. We further assume that the vector fields $f_p:\mathcal{X}_p\times\mathcal{D}\to\mathcal{X}_p$ are locally Lipschitz.

Then, Theorem~\ref{thm:dist-agnostic} below shows that the switching conditions stated in the absence of disturbances in Section~\ref{subsec:switch-system} guarantee that the switched system \eqref{eq:switch-system-disturb} will be safe according to \eqref{eq:safety_criterion} despite the presence of disturbances. 

\begin{theorem}\label{thm:dist-agnostic}
Consider \eqref{eq:switch-system-disturb} where $\sigma$ is a switching signal and $\sigma(k) = p \in \mathcal{P}$. Assume that
\begin{enumerate}[(i)]
\item for each $p\in\mathcal{P}$, $x^*_p$ is an exponentially stable equilibrium of $f_p$ in the absence of disturbances, let $V_p$ be a locally Lipschitz Lyapunov function and $\mathcal{B}_p$ an estimate of the BoA based on $V_p$;
%
\item there exists a $\kappa>0$, $\overline{N}_0>0$ such that \eqref{eq:feasible-0-inp} holds.
\end{enumerate}
Then, there exists a $\delta>0$ such that for any disturbance $\{d_k\}_{k\in\mathbb{Z}_+}$ with $\|d\|_\infty<\delta$, and for any switching signal $\sigma$ that satisfies Definition~\ref{def:avg-dwell} with $N_0 \leq \overline{N}_0$ and $N_{\rm a}\geq \overline{N}_{\rm a}$ defined in \eqref{eq:avg-dwell}, the solution of \eqref{eq:switch-system-disturb} satisfies
\begin{align}\label{eq:target-set}
x_0 \in \bigcap_{p\in\mathcal{P}} \mathcal{M}_p(\omega(\kappa)) \Rightarrow x_k \in  \mathcal{M}(\bar{\omega}(\|d\|_\infty)) \subset \bigcap_{p\in\mathcal{P}} \mathcal{B}_p \enspace,
\end{align}
for all $k\in\mathbb{Z}_+$, where $\bar{\omega}(\|d\|_\infty):=\mu(\kappa)^{N_0}\omega(\kappa) + \alpha(\|d\|_\infty)$ and\footnote{Terminology: A function $\alpha:\mathbb{R}_+ \to \mathbb{R}_+$ belongs to $\mathcal{K}_\infty$ if it is continuous, strictly increasing, $\alpha(0)=0$, and $\lim_{s\to\infty} \alpha(s)=\infty$. } $\alpha\in\mathcal{K}_\infty$.
\end{theorem}

The proof of Theorem~\ref{thm:dist-agnostic} is detailed in the appendix.


\section{Adaptation of a 3D Limit-Cycle Biped}
\label{sec:example}

In this section we apply the proposed switching framework to adapt the walking gaits of a 3D limit-cycle biped to an externally applied force. 

\subsection{Task and Robot Model}

Our motivation stems from collaborative object transportation tasks, in which a leading co-worker---either a robot or a human---physically interacts with a biped to direct its motion. It is assumed that the leader's intention can be represented as a sufficiently smooth trajectory $p_{\rm L}(t)$. Although the biped does not know $p_{\rm L}(t)$ explicitly, it experiences an interaction force $F_{\rm e}(t)$ applied by the collaborator, which carries information about the intended trajectory. Following~\cite{rahman2002investigation, tsuji2004analysis}, to model this interaction the leader's intention $p_{\rm L}(t)$ is translated to the force $F_{\rm e}(t)$ using an impedance model 
\begin{equation}\nonumber
F_{\rm e}(t) = K_{\rm L}(p_{\rm L}(t)-p_{\rm E}(t)) + N_{\rm L}(\dot{p}_{\rm L}(t)-\dot{p}_{\rm E}(t)) \enspace, 
\end{equation}
where $p_{\rm E}$ is the point at which the force is applied and $K_{\rm L}$ and $N_{\rm L}$ are suitable stiffness and damping matrices, respectively; see \cite{motahar2015impedance, motahar2017steering, veer2017supervisory} for more details. 

The bipedal robot model employed here is shown in Fig.~\ref{fig:model} and is similar to models that have appeared in the literature~\cite{3D-robotica2012, motahar2017steering}; thus, the exposition below will be terse. The model features nine degrees of freedom, and its configuration can be described by the coordinates $q:=(q_1,q_2,...,q_9)$ as in Fig.~\ref{fig:model}. It is assumed that all degrees of freedom other than yaw $q_1$ and pitch $q_2$ are actuated.  
%
The walking cycle is composed of alternating sequences of single and double support phases. As in~\cite{3D-robotica2012}, we assume that the double support phases are instantaneous and can be modeled as an impact event based on the hypotheses listed in~\cite[Chapter 3]{westervelt2007feedback}. Defining $\hat{x}:=(q,\dot{q})$ as the state of the robot, walking can be represented as a system with impulse effects
\begin{align}
\Sigma:
\begin{cases}
~~\dot{\hat{x}} = f(\hat{x}) + g(\hat{x})u + g_{\rm e}(\hat{x}) F_{\rm e}, & \mathrm{if}~\hat{x}\not\in\mathcal{S} \\
~~\hat{x}^+ = \Delta(\hat{x}^-), & \mathrm{if}~\hat{x}^-\in\mathcal{S}
\end{cases} \enspace, \label{eq:SIE-robot}
\end{align}
where $u$ are the inputs, $(f, g, g_{\rm e})$ describe the swing phase dynamics in the presence of the external force $F_{\rm e}$. In~\eqref{eq:SIE-robot}, $\mathcal{S}$ represents the ground surface and $\Delta$ is a mapping taking the states $\hat{x}^-$ prior to impact to the states $\hat{x}^+$ right after impact. More details on the model can be found in~\cite{3D-robotica2012} with the difference due to the existence of the external force, which is taken into account as explained in~\cite{motahar2017steering}.   

\begin{figure}[t!]
\begin{centering}
\includegraphics[width=0.55\columnwidth]{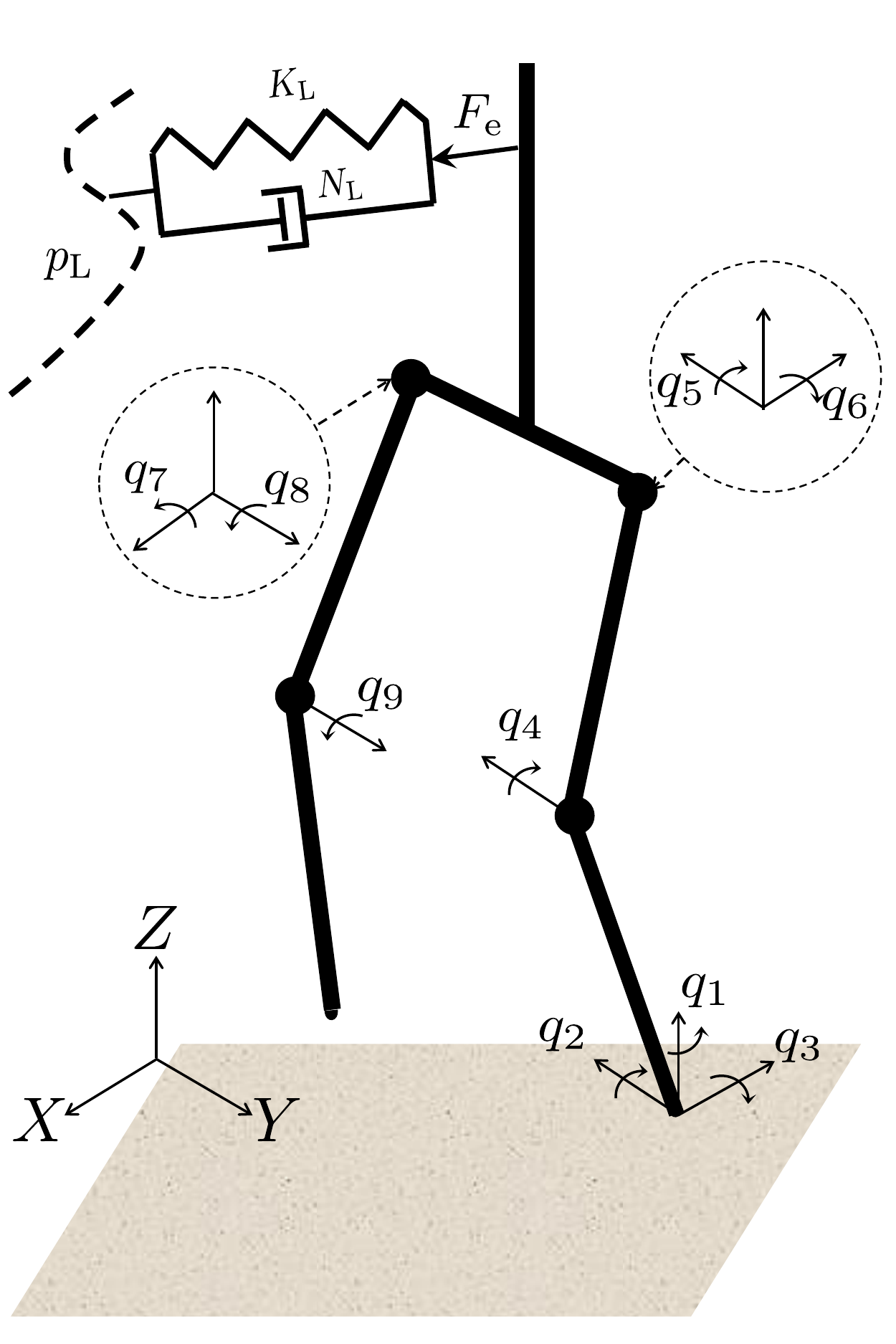} 
\par\end{centering}
\vspace{-0.1in}
\caption{Robot model with a choice of generalized coordinates when supported on left leg.}
\vspace{-0.3in}
\label{fig:model} 
\end{figure}

%

\subsection{Family of Controllers}
\label{subsec:control-fam}

We work with a finite family of feedback controllers $u=\Gamma_p(x, F_{\rm e})$, where $x$ includes all components of $\hat{x}$ except $q_1$, i.e. $\hat{x}:=(q_1,x)$. Each controller results in a limit cycle $\mathcal{O}_p$ that corresponds to a straight-line or a turning walking motion. The controllers are designed within the HZD framework as in \cite{motahar2017steering, 3D-robotica2012}, and they guarantee exponential stability of $\mathcal{O}_p$ in the absence of the external force. 
%
In the interest of space, we will only discuss some important properties of the controllers $\Gamma_p$; details can be found in \cite{motahar2017steering}. Associated to each $\Gamma_p$ is a zero dynamics surface $\mathcal{Z}_p$ which, for each $p\in\mathcal{P}$, has the following properties:
\begin{enumerate}[(i)]
\item $\mathcal{Z}_p$ is invariant under the biped's closed-loop dynamics despite the presence of $F_{\rm e}$; 
%
\item for each $\mathcal{Z}_p$, we have $\mathcal{S}\cap\mathcal{Z}_p = \mathcal{S}\cap\mathcal{Z}_q$ for all $p,q\in\mathcal{P}$. 
\end{enumerate}
A point of clarification is due here. The properties listed above are specific to the controllers used in this example; any other control method that generates exponentially stable walking limit cycles could be used without changing the procedure of Section~\ref{sec:switch-guarantees} for safe switching. However, these properties vastly simplify computation of the Lyapunov functions and the corresponding estimates of the BoAs involved in establishing the conditions of Section~\ref{sec:switch-guarantees} due to the associated dimensional reduction.

\subsection{Limit-Cycle Gait Primitives and Switching Among Them}

The behavior of the limit cycle walking motions $\mathcal{O}_p$ can be studied via the corresponding \emph{forced} Poincar\'e map~\cite{veer2017poincare}
\begin{equation}\label{eq:discrete-dyn-xhat}
\hat{x}_{k+1} = \hat{P}_p(\hat{x}_k, F_{{\rm e},k}) \enspace,
\end{equation}
where $F_{{\rm e}, k}$ is the externally applied force over one stride. Although the forced Poincar\'e map \eqref{eq:discrete-dyn-xhat} is 17-dimensional, thus making the estimation of the BoA very challenging, the controller properties listed in Section~\ref{subsec:control-fam} allow us to work with a drastically lower-dimensional system. Due to the invariance of the associated zero dynamics surfaces $\mathcal{Z}_p$, the state always returns to $\mathcal{S}\cap\mathcal{Z}_p\subset\mathbb{R}^3$ at the end of each step. Taking into account that rotations around the yaw axis do not alter the dynamics---i.e., the system is equivariant to $q_1$ as shown in~\cite{3D-robotica2012, motahar2016composing}---allows us to work with the two-dimensional restricted forced Poincar\'e map
\begin{align}
z_{k+1} & = \rho_p(z_k,F_{{\rm e},k}) \enspace, \label{eq:discrete-dyn-z}
\end{align}
where $z$ are suitable coordinates on $\mathcal{S}\cap\mathcal{Z}_p$. As a result, the limit cycle walking motions $\mathcal{O}_p$ can be represented as motion primitives of the form $\mathcal{G}_p:=\{\rho_p,z_p^*\}$ where $z_p^*$ is a fixed point of $\rho_p$. Switching among these primitives according to a signal $\sigma(k)$ can be described by the switched system
\begin{equation}\label{eq:switch-system-zd}
z_{k+1} = \rho_{\sigma(k)}(z_k,F_{{\rm e},k}) \enspace.
\end{equation}
It is worth reminding that the dimensional reduction afforded by HZD resulted in a 2-dimensional switched systems instead of the original 17-dimensional system. 

\subsection{Simulation Results}
\label{sec:sim-results}


Working within the family of controllers discussed in Section~\ref{subsec:control-fam}, we generate three gait primitives: $\mathcal{G}_0$ to turn clockwise by $30^{\circ}$, $\mathcal{G}_1$ to walk straight, and $\mathcal{G}_2$ to turn counter-clockwise by $30^{\circ}$. Using SoS programming, $\mathcal{B}_p$'s are obtained and plotted in Fig.~\ref{fig:BoA} as dashed ellipses; exact details of the SoS program can be found in \cite{motahar2016composing}. 
We choose $\kappa=0.002$ and compute upper bounds for $\mu(\kappa)$ and $\omega(\kappa)$ using \cite[Proposition~1]{veer2018ultimate}. Using these we obtain $\overline{\mathcal{M}}$ defined in \eqref{eq:feasible-0-inp} and it can be noted from Fig.~\ref{fig:BoA}, that this set lies within $\accentset{\circ}{\mathcal{B}}_0\cap\accentset{\circ}{\mathcal{B}}_1 \cap\accentset{\circ}{\mathcal{B}}_2$ for $\overline{N}_0=2$; further, computing \eqref{eq:avg-dwell} gives $\overline{N}_{\rm a}=0.99$. With this choice of $\overline{N}_0$ and $\overline{N}_{\rm a}$, arbitrary switching signals satisfy \eqref{eq:avg-dwell-time-def}. Hence, by Theorem~\ref{thm:dist-agnostic}, the biped can switch to a different primitive each stride without compromising its safety, despite the external force.

\begin{figure}[t]
\vspace{-0.05in}
\begin{centering}
\includegraphics[width=0.9\columnwidth]{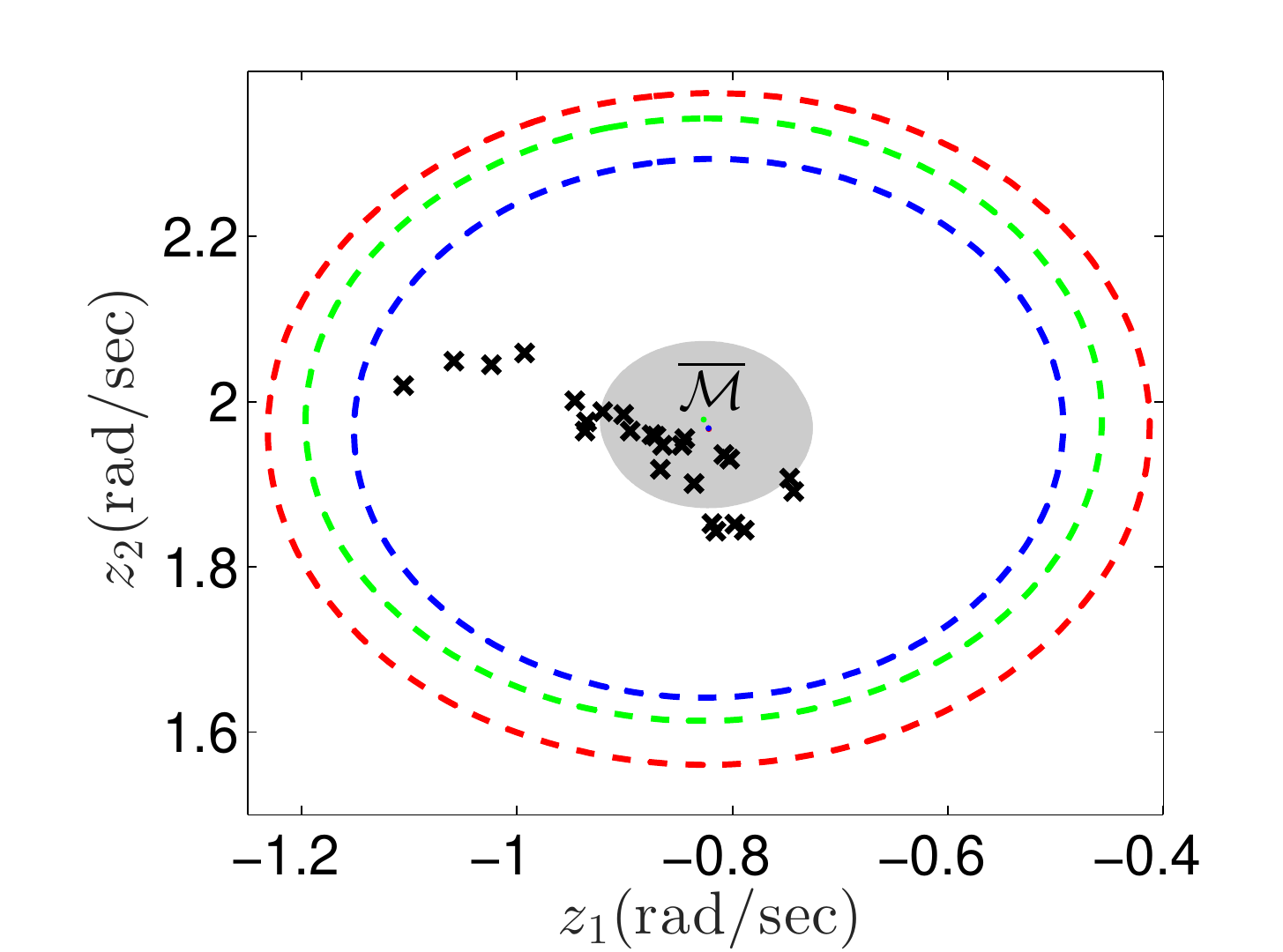}
\par\end{centering}
\vspace{-0.1in}
\caption{Estimates of the BoA for the primitives in $\mathbb{G}$ and verification of \eqref{eq:feasible-0-inp}. The BoA estimates $\mathcal{B}_0$, $\mathcal{B}_1$, and $\mathcal{B}_2$ are the dashed red, green, and blue ellipses, respectively. The grey region is $\overline{\mathcal{M}}$ defined in \eqref{eq:feasible-0-inp} for $\kappa=0.002$, $\overline{N}_0=2$. Black crosses are the solution of \eqref{eq:switch-system-zd} for the simulation in Fig.~\ref{fig:leader_follow}. }
\vspace{-0.25in}
\label{fig:BoA} 
\end{figure}

\begin{figure}[b]
\vspace{-0.25in}
\begin{centering}
\includegraphics[width=0.78\columnwidth]{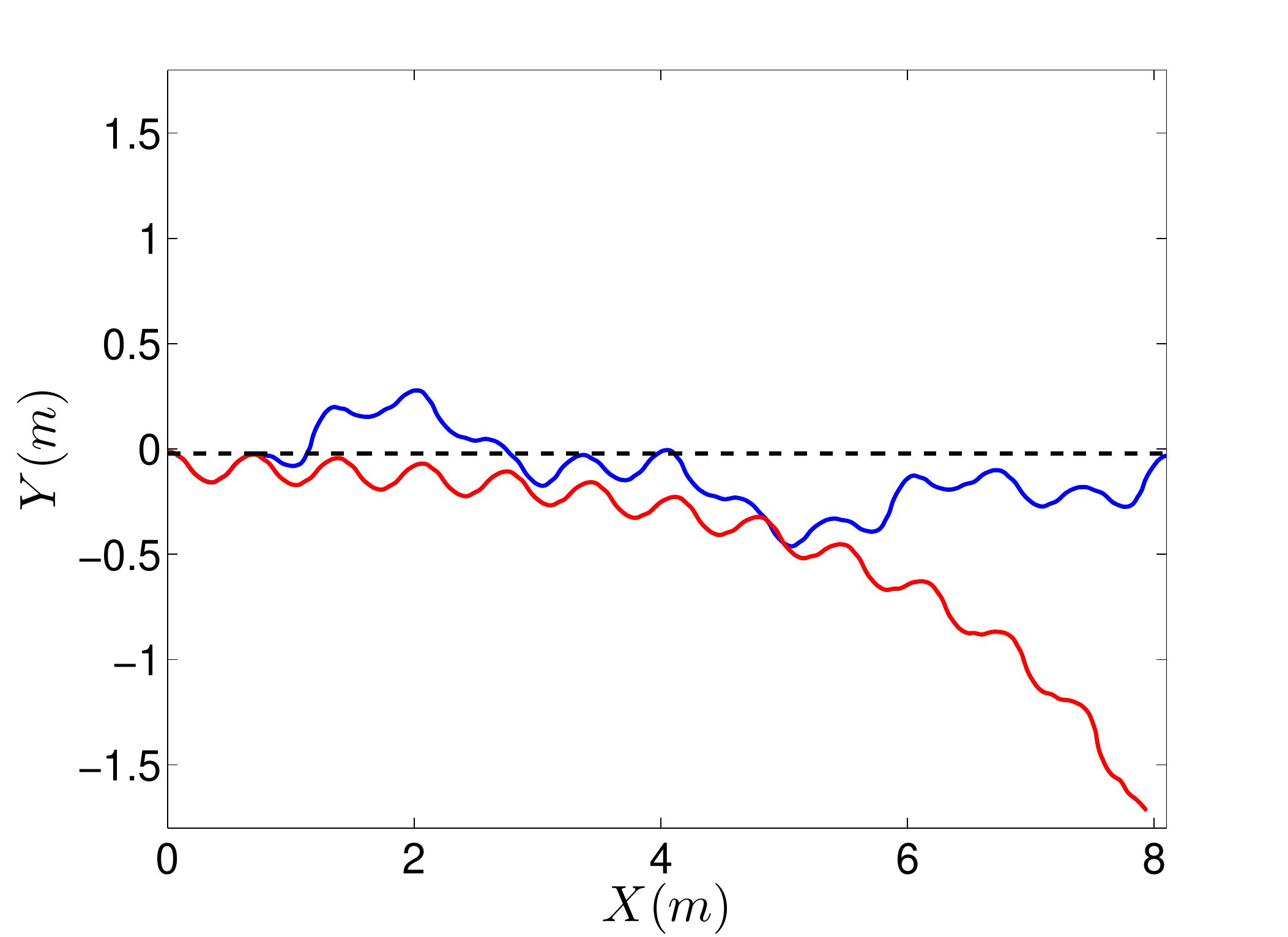} 
\par\end{centering}
\vspace{-0.15in}
\caption{Bipedal walker collaborating with a leader walking straight at 0.65 m/s. The dashed black line is the leader's intended trajectory. The trajectory of the biped's point-of-application of the force is denoted by red when only $\mathcal{G}_0$ is executed; and by blue when adaptive switching is implemented.}
\vspace{-0.1in}
\label{fig:biped_lead_straight} 
\end{figure}

Now we turn our attention towards the switching policy that adapts the biped's gait to the leader's intended trajectory $p_{\rm L}(t)$, which is not directly available to the biped, requiring the planner to harness the external force feedback as a cue for adaptation. Our switching policy estimates the ``average" heading direction $\Phi_k$ that the force is pointing towards over a stride, and then chooses the primitive that turns the biped towards this estimated heading. To compute $\Phi_k$, we integrate the force along the $X$ and $Y$ directions over a stride; see Fig.~\ref{fig:model} for the global coordinate frame. Let $t_0=0$ be the initial time and $t_k$ be the time at the end of the $k$-th stride. Then, over the $(k+1)$-th stride, the integral of the force components are
\begin{align}\nonumber
F_k^X:=\int_{t_k}^{t_{k+1}} F_{\rm e}^X(t)~dt, & & F_k^Y:=\int_{t_k}^{t_{k+1}} F_{\rm e}^Y(t)~dt \enspace,
\end{align}
which are used to compute the ``average" heading as $\Phi_k=\arctan(F_k^Y/F_k^X)$. The switching policy is chosen to be 
\begin{align}
\sigma(k+1)= \mathrm{sign}(\Phi_k)+1 \enspace, \label{eq:policy}
\end{align}
where the $\mathrm{sign}$ function returns -1, 0, 1 for negative, 0, and positive $\Phi_k$, respectively. It can be observed from the switching policy that there is a one-step time delay in response to the force, i.e., $\Phi_k$ is used to obtain $\sigma(k+1)$. 

We first test our planning framework for walking straight. It can be noted from Fig.~\ref{fig:biped_lead_straight} that with a single primitive $\mathcal{G}_0$ the biped drifts away from $p_{\rm L}(t)$. However, the switching framework is able to adapt the biped's gait to keep it within the vicinity of the leader's intended trajectory. Next, we simulate a more complex scenario shown in Fig.~\ref{fig:leader_follow} where $p_{\rm L}(t)$ is represented by the red line, along which the leader intends to move at a constant speed of 0.65 m/s. Following the switching signal generated by the supervisor, the biped is able to adapt to the leader's intended trajectory while maintaining its safety, as verfied by Fig.~\ref{fig:BoA} where the solution of the switched system \eqref{eq:switch-system-zd}, denoted by black crosses, lies within $\bigcap_{p\in\mathcal{P}} \mathcal{B}_p$, satisfying \eqref{eq:safety_criterion}. 
Lastly, note that the biped naturally adapts its speed to the external force without requiring the planner's intervention; see \cite{veer2015adaptation,motahar2017steering} for details. 


\begin{figure}[b]
\vspace{-0.2in}
\begin{centering}
\includegraphics[width=0.95\columnwidth]{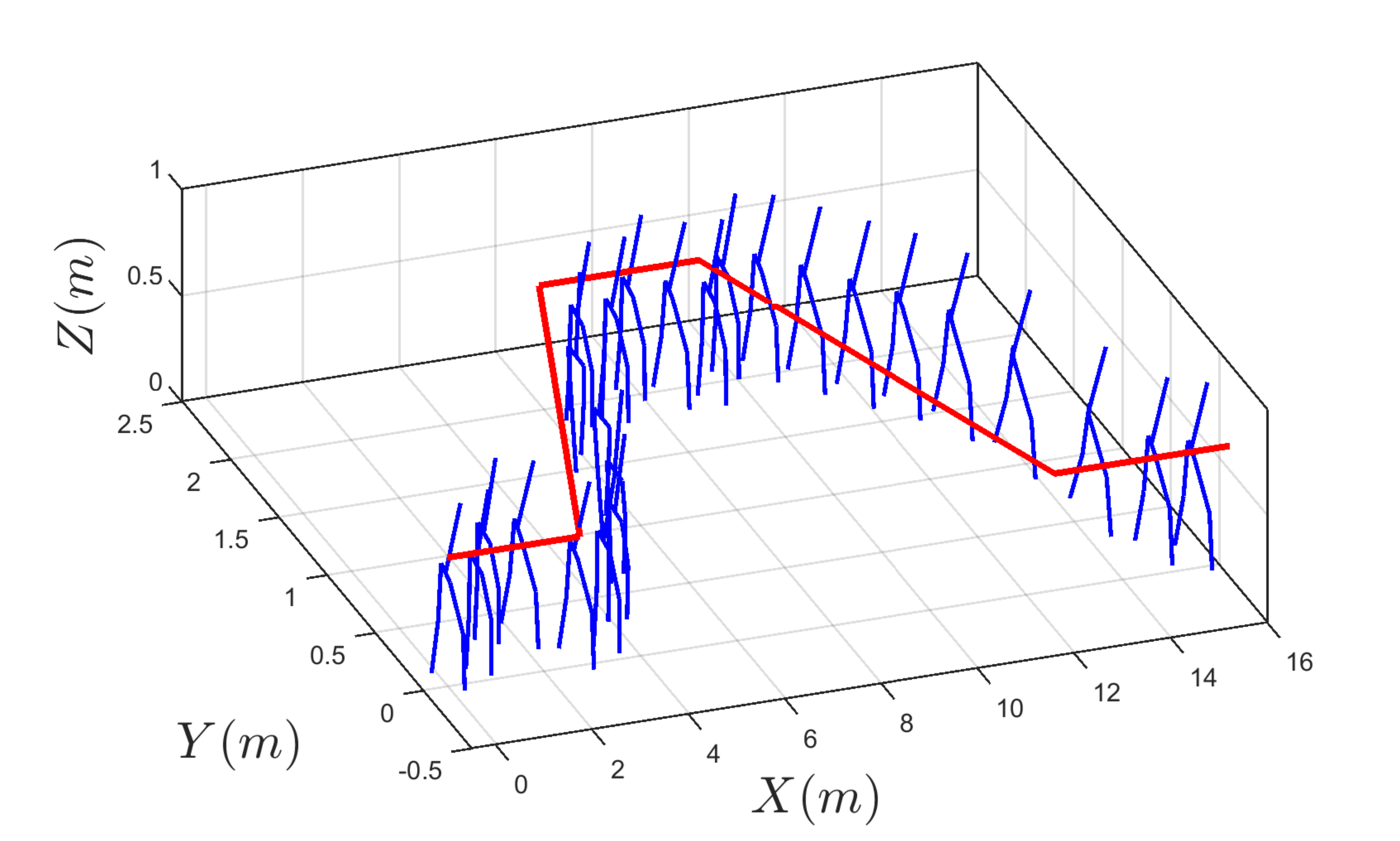} 
\par\end{centering}
\vspace{-0.15in}
\caption{Biped adapting to follow the leader's intended trajectory $p_{\rm L}(t)$ by switching among the primitives $\mathcal{G}_0.\mathcal{G}_1,\mathcal{G}_2$. The red line represents $p_{\rm L}(t)$ and the blue stick figures represent the biped at the end of alternate steps.}
\vspace{-0.1in}
\label{fig:leader_follow} 
\end{figure}

\section{Conclusion}
\label{sec:conclusion}

In this paper we proposed a general planning framework that facilitates adaptive planning in the face of uncertainty and external signals. Our framework views switching among externally excited dynamic primitives as a switched system with multiple equilibria under disturbances. We analyze this switched system to extract an average dwell-time bound that succinctly captures the dynamic limitations of the robot, resulting in motion plans consistent with the dynamics. The planning framework is particularized to a biped-leader collaborative task, where the biped's gait is adapted to follow the leader's intended trajectory, harnessing the interaction force as a command.

\appendix
For convenience, denote the space of uniformly bounded sequences in $\mathcal{D}$ by $l_\infty^{\mathcal{D}}$; and $\mathcal{B}_p$ as the $\overline{\kappa}_p>0$ sublevel set of $V_p$, i.e., $\mathcal{B}_p=\{x\in\mathbb{R}^n~|~V_p(x)\leq \overline{\kappa}_p\}$. 

Before proceeding to the proof of Theorem~\ref{thm:dist-agnostic}, we first establish the following lemma, which shows that a Lyapunov function on a compact set of the state space is also an ISS-Lyapunov function on the same set, provided, that the disturbances are sufficiently small. 
\begin{lemma}\label{lem:ISS-Lyap}
Let $V_p:\mathcal{X}_p \to \mathbb{R}_+$ be a Lyapunov function for $f_p(\cdot,0)$ that satisfies \eqref{eq:V-2-0-inp}. Then, there exist a\footnote{We assume that $\delta$ is the same for each $p\in\mathcal{P}$; if this is not the case we can choose $\delta:=\min_{p\in\mathcal{P}} \delta_p$ without loss of generality. } $\delta>0$ and $\alpha_p\in\mathcal{K}_\infty$, such that $V_p$ is an ISS-Lyapunov function, that satisfies for all $x\in\mathcal{B}_p$ and $d:=\{d_k\}_{k\in\mathbb{Z}_+}\in l_\infty^{\mathcal{D}}$ with $\|d\|_\infty<\delta$,
\begin{equation}\nonumber
V_p(f_p(x,d_k)) \leq \lambda V_p(x) + \alpha_p(\|d\|_\infty)
\end{equation}
where $0<\lambda<1$ is the same as in \eqref{eq:V-2-0-inp}.
\end{lemma}

\begin{proof}
This proof builds on \cite[Theorem~2]{veer2017poincare}, from which it follows that a Lyapunov function $V_p$ is also an ISS-Lyapunov function in a set where $g_p:=V_p\circ f_p:\mathcal{X}_p\times\mathcal{D}\to\mathbb{R}_+$ has a uniform Lipschitz constant. Hence, we first establish this property in our region of interest.

\vspace{0.5mm}
\noindent \emph{Claim 1:} There exists a $\delta>0$ such that $g_p$ is Lipschitz for all $x\in\mathcal{B}_p$ and\footnote{Notation: We use $B_\delta(a)$ to denote an open-ball of radius $\delta$ centered at $a$. This notation can be used for open-balls in $\mathbb{R}^n$, as well as $\mathcal{D}$. It will be clear from context the space to which the ball belongs.} $d\in B_{\delta}(0)$.\\
The proof for Claim 1 is detailed after the proof of Theorem~\ref{thm:dist-agnostic}. It is worth noting that if $\mathcal{D}$ were a finite-dimensional Banach space, $\mathcal{B}_p\times\overline{B}_\delta(0)$ would be a compact set by Heine-Borel theorem, and then Claim 1 would follow from the fact that locally Lipschitz functions are Lipschitz on compact sets. However, as we allow for an infinite-dimensional Banach space, more technical care is required. 

With Claim 1, we can merely repeat the proof of \cite[Theorem~2]{veer2017poincare} to complete the proof of Lemma~\ref{lem:ISS-Lyap}.
\end{proof}


Now we are ready to present the proof of Theorem~\ref{thm:dist-agnostic}.

\begin{proof}\emph{[Theorem~\ref{thm:dist-agnostic}]}
The proof of this theorem utilizes \cite[Corollary~2]{veer2018ultimate} which shows that for switched systems with \emph{global} ISS subsystems, if the switching signal satisfies the average dwell-time bound \eqref{eq:avg-dwell}, then, there exists a non-empty compact set $\mathcal{C}\subset\mathcal{M}(\mu(\kappa)^{\overline{N}_0}\omega(\kappa))$ such that for all $k\in\mathbb{Z}_+$,
\begin{equation}\label{eq:cor-use}
x_0\in\mathcal{C} \implies x_k\in \mathcal{M}(\bar{\omega}(\|d\|_\infty)) \enspace.
\end{equation}
The proof of \cite[Corollary~2]{veer2018ultimate} explicitly characterizes $\mathcal{C}$ as $\mathcal{C}=\bigcap_{p\in\mathcal{P}}\mathcal{M}_p(\omega(\kappa))$. With this, and \eqref{eq:cor-use}, we can obtain \eqref{eq:target-set}. However, \cite[Corollary~2]{veer2018ultimate} requires global ISS-Lyapunov functions for each subsystem, but the $V_p$ from Lemma~\ref{lem:ISS-Lyap} are ISS-Lyapunov functions only for $x\in\mathcal{B}_p$ and $d\in\mathcal{D}$ with $\|d\|_\infty<\delta$. Therefore, to use this corollary we restrict the solutions of \eqref{eq:switched-system} to evolve in the domain $\bigcap_{p\in\mathcal{P}}\mathcal{B}_p$ for $d\in\mathcal{D}$ with $\|d\|_\infty<\delta$, where each ISS-Lyapunov function is valid. 

To ensure that the solutions stay within $\bigcap_{p\in\mathcal{P}}\mathcal{B}_p$, we claim that it is sufficient to show that for some $\delta>0$,
\begin{equation}\label{eq:desired-relation}
\mathcal{M}(\bar{\omega}(\delta)) \subset \bigcap_{p\in\mathcal{P}} \accentset{\circ}{\mathcal{B}}_p \enspace.
\end{equation}
By the monotonicity of the sublevel sets, if $\|d\|_\infty<\delta$, $\mathcal{M}(\bar{\omega}(\|d\|_\infty))\subset \mathcal{M}(\bar{\omega}(\delta))$, which on using in \eqref{eq:desired-relation} gives $\mathcal{M}(\bar{\omega}(\|d\|_\infty)) \subset \bigcap_{p\in\mathcal{P}} \accentset{\circ}{\mathcal{B}}_p$. Hence, by \eqref{eq:cor-use} the solutions would be restricted to $\bigcap_{p\in\mathcal{P}} \accentset{\circ}{\mathcal{B}}_p$ for all $k\in\mathbb{Z}_+$ and we can obtain \eqref{eq:target-set}. In what follows, we will show that there exists a $\delta>0$, such that for disturbances smaller than $\delta$, \eqref{eq:desired-relation} holds.

For the sake of convenience, let $\underline{\mathcal{B}}:=\bigcup_{p\in\mathcal{P}} \mathcal{B}_p\setminus \bigcap_{p\in\mathcal{P}} \accentset{\circ}{\mathcal{B}}_p$ which does not contain any $x\in\mathcal{M}(\mu(\kappa)^{\overline{N}_0}\omega(\kappa))$ as $\mathcal{M}(\mu(\kappa)^{\overline{N}_0}\omega(\kappa))$ is in the interior of $\bigcap_{p\in\mathcal{P}} \mathcal{B}_p$. Therefore, 
\begin{align}\label{eq:boundary-Vp}
\forall p\in\mathcal{P}, & ~~~ V_p(x)>\mu(\kappa)^{\overline{N}_0}\omega(\kappa), & \forall x\in  \underline{\mathcal{B}} \enspace.
\end{align}
Let $\overline{\kappa}$ be defined as
\begin{equation}\label{eq:kappa-min-def}
\overline{\kappa}:=\min_{p\in\mathcal{P}} \min_{x\in \underline{\mathcal{B}}} V_p(x) \enspace,
\end{equation}
which is well-defined because $ \underline{\mathcal{B}}$ is compact and $\mathcal{P}$ is finite.
From \eqref{eq:boundary-Vp}, it follows that $\overline{\kappa}>\mu(\kappa)^{\overline{N}_0}\omega(\kappa)$. Let $0< c < \overline{\kappa} - \mu(\kappa)^{\overline{N}_0}\omega(\kappa)$, and shrink $\delta$ if necessary to ensure $0<\delta < \alpha^{-1}(c)$. Then, for any $p\in\mathcal{P}$, and $x\in\mathcal{M}_p(\bar{\omega}(\delta))$, 
\begin{align}
V_p(x) & \leq \mu(\kappa)^{\overline{N}_0}\omega(\kappa) + \alpha(\delta) \nonumber\\
& < \mu(\kappa)^{\overline{N}_0}\omega(\kappa) + c < \overline{\kappa} \enspace. \label{eq:Vp<kappa-boundary}
\end{align}
Furthermore, with this choice of $\delta$,
\begin{equation}\label{eq:Mp-in-Bp}
\mathcal{M}_p(\bar{\omega}(\delta))\subset \mathcal{B}_p \enspace.
\end{equation}
To see this, let $x\in\mathcal{M}_p(\bar{\omega}(\delta))$, then\footnote{Reminder: $\mathcal{B}_p$ is the $\overline{\kappa}_p$ sublevel set of $V_p$.} $V_p(x)<\overline{\kappa}\leq \overline{\kappa}_p$ by \eqref{eq:Vp<kappa-boundary} and $\overline{\kappa}\leq \min_{p\in\mathcal{P}}\overline{\kappa}_p$. It can be observed that $\overline{\kappa}\leq \min_{p\in\mathcal{P}}\overline{\kappa}_p$ because if not, then by the definition of $\overline{\kappa}$ in \eqref{eq:kappa-min-def}, for each $p\in\mathcal{P}$ and $x\in\underline{\mathcal{B}}$, $V_p(x)\geq \overline{\kappa} >\min_{p\in\mathcal{P}}\overline{\kappa}_p$, implying that every point in $\underline{\mathcal{B}}$ is strictly ``outside" some $\mathcal{B}_p$, resulting in a contradiction with the definition of $\underline{\mathcal{B}}$ that contains elements of each $\mathcal{B}_p$. This follows from the fact that $\mathcal{B}_p$ is closed, hence its boundary $\partial\mathcal{B}_p$ is a subset of $\mathcal{B}_p$, but by the definition of the boundary of a set, $\partial \mathcal{B}_p$ is not in $\accentset{\circ}{\mathcal{B}}_p$, hence $\underline{\mathcal{B}}$ includes each $\partial \mathcal{B}_p\subset \mathcal{B}_p$.

Finally, we claim that with the given choice of $\delta$ as above, \eqref{eq:desired-relation} holds. To check this claim, assume that \eqref{eq:desired-relation} does not hold for this choice of $\delta$. Then, there must exist a $\hat{x}\in \mathcal{M}(\bar{\omega}(\delta))$ which lies in the complement of $\bigcap_{p\in\mathcal{P}} \accentset{\circ}{\mathcal{B}}_p$. There exists a $p\in\mathcal{P}$ for which $\hat{x}\in\mathcal{M}_p(\bar{\omega}(\delta))$, which by \eqref{eq:Mp-in-Bp} gives $\hat{x}\in\mathcal{B}_p$ but $\hat{x}\not\in\bigcap_{p\in\mathcal{P}} \accentset{\circ}{\mathcal{B}}_p$, further implying that $\hat{x}\in\underline{\mathcal{B}}$. Hence, by the definition of $\overline{\kappa}$ in \eqref{eq:kappa-min-def}, $V_p(\hat{x})\geq \overline{\kappa}$. On the other hand, as $\hat{x}\in\mathcal{M}_p(\bar{\omega}(\delta))$, by \eqref{eq:Vp<kappa-boundary} it follows that $V_p(\hat{x})< \overline{\kappa}$, leading to a contradiction with $V_p(\hat{x})\geq \overline{\kappa}$. Hence, there exists a $\delta>0$ for which \eqref{eq:desired-relation} holds, completing the proof of Theorem~\ref{thm:dist-agnostic}.
\end{proof}

\begin{proof}\emph{[Claim 1 in Lemma~\ref{lem:ISS-Lyap}]}
As $V_p$ and $f_p$ are locally Lipschitz in their arguments, their composition $g_p:=V_p\circ f_p$ is locally Lipschitz as well. Hence, for any $(x,0)\in \mathcal{B}_p \times \mathcal{D}$, there exists a $\delta_x>0$ and $L_x>0$ such that $\|g_p(x_1,d_1)-g_p(x_2,d_2)\|\leq L_x \| (x_1-x_2,d_1-d_2)\| $, for any $x_1,x_2 \in B_{\delta_x}(x)$ and $d_1,d_2 \in B_{\delta_x}(0)$. Construct an open cover $\bigcup_{x\in\mathcal{B}_p} B_{\delta_x/2}(x)$ of $\mathcal{B}_p$ which is compact, hence there exists $\hat{x}_1,~\hat{x}_2,\cdots,\hat{x}_N$ such that $\mathcal{B}_p \subset \bigcup_{i=1}^n B_{\delta_i/2}(\hat{x}_i)$ where $\delta_i:=\delta_{\hat{x}_i}$ and define $\delta:=\min\{\delta_1/2,\cdots,\delta_N/2\}$.

Consider $x_1,x_2\in\mathcal{B}_p$ and $d_1,d_2\in B_\delta(0)$. Then, the following two cases arise.

\vspace{1mm}
\noindent \emph{Case (a):} There exists an $i\in\{1,\cdots,N\}$ such that $x_1,x_2\in B_{\delta_i}(\hat{x}_i)$. \\
As $d_1,d_2\in B_\delta(0)\subset B_{\delta_i/2}(0)$, and by the assumption of this case $x_1,x_2\in B_{\delta_i}(\hat{x}_i)$, we can use the Lipschitz continuity of $g_p$ in the $\delta_i$ neighborhood of $(\hat{x}_i,0)$ to obtain $\|g_p(x_1,d_1)-g_p(x_2,d_2)\|\leq L_i \| (x_1-x_2,d_1-d_2)\| $ where $L_i:=L_{\hat{x}_i}$. Define $\hat{L}:=\max\{L_1,\cdots,L_N \}$, then we can express the Lipschitz bound as
\begin{equation}\label{eq:lipschitz-case-a}
\|g_p(x_1,d_1)-g_p(x_2,d_2)\|\leq \hat{L} \| (x_1-x_2,d_1-d_2)\| \enspace.
\end{equation}

\vspace{1mm}
\noindent \emph{Case (b):} There does not exist any $i\in\{1,\cdots,N\}$ such that $x_1,x_2\in B_{\delta_i}(\hat{x}_i)$. \\
To obtain the Lipschitz bound in this case we first need to establish uniform boundedness of $g_p$ over $\mathcal{B}_p\times B_\delta(0)\subset\mathcal{X}_p\times\mathcal{D}$. Note that $g_p(\cdot,0):\mathcal{X}_p\to\mathcal{X}_p$ is Lipchitz on the compact set $\mathcal{B}_p$ as it is locally Lipshcitz in its arguments. Hence, there exists a $\tilde{L}>0$ such that $\|g_p(y_1,0)-g_p(y_2,0)\|\leq \tilde{L} \|y_1-y_2\|$ for any $y_1,y_2\in \mathcal{B}_p$. Further, using the boundedness (compactness) of $\mathcal{B}_p \subset \mathbb{R}^n$, there exists a $r>0$ such that $\|y_1-y_2\|\leq r$ for any $y_1,y_2\in\mathcal{B}_p$. As $\mathcal{B}_p \subset \bigcup_{i=1}^n B_{\delta_i/2}(\hat{x}_i)$, there exist $\hat{x}_m$ and $\hat{x}_n$ such that $\|x_1-\hat{x}_n\|<\delta_n/2$ and $\|x_2-\hat{x}_m\|<\delta_m/2$. Then, 
\begin{align}
& \|g_p(x_1,d_1)- g_p(x_2,d_2)\| \nonumber \\
& = \|g_p(x_1,d_1)-g_p(\hat{x}_n,0) + g_p(\hat{x}_n,0) - g_p(\hat{x}_m,0) \nonumber \\
& ~~~ + g_p(\hat{x}_m,0) - g_p(x_2,d_2) \| \nonumber \\
& \leq \|g_p(x_1,d_1)-g_p(\hat{x}_n,0)\| + \|g_p(\hat{x}_n,0) - g_p(\hat{x}_m,0)\| \nonumber \\
& ~~~ + \|g_p(\hat{x}_m,0) - g_p(x_2,d_2) \| \nonumber \\
& \leq L_n \big( \|x_1-\hat{x}_n\| + \|d_1\| \big) + \tilde{L} \|\hat{x}_n-\hat{x}_m\| \nonumber \\
& ~~~ + L_m \big( \|x_2-\hat{x}_m\| + \|d_2\| \big) \nonumber \\
& \leq 2\hat{L} \big( r+\delta \big) + \tilde{L}r =: M \enspace. \label{eq:P-uniform-bound}
\end{align}
Also, it can be noted that $\|x_1-x_2\|\geq \delta$ which can be shown by the way of contradiction. Suppose $\|x_1-x_2\| < \delta$. Let $\hat{x}_n$ be such that $\|x_1-\hat{x}_n\|<\delta_n/2$ which exists because $\mathcal{B}_p \subset \bigcup_{i=1}^n B_{\delta_i/2}(\hat{x}_i)$. Then, adding and subtracting this $\hat{x}_n$ in $\|x_1-x_2\|$, and using reverse triangle inequality gives
\begin{align}
\|x_2-\hat{x}_n\|-\|x_1-\hat{x}_n\| & \leq\|x_1-\hat{x}_n+\hat{x}_n-x_2\| <\delta  \nonumber 
\end{align}
which leads to $\|x_2-\hat{x}_n\| < \delta + \|x_1-\hat{x}_n\| < \delta_n/2 + \delta_n/2=\delta_n$ implying that $x_2\in B_{\delta_n}(\hat{x}_n)$, which along with the fact that $x_1\in B_{\delta_n}(\hat{x}_n)$ leads to a contradiction with the assumption of Case (b). Hence, $\|x_1-x_2\|\geq \delta$ which is used in \eqref{eq:P-uniform-bound} to obtain
\begin{align}
& \|g_p(x_1,d_1)- g_p(x_2,d_2)\| \nonumber \\
& \leq M \leq \frac{M}{\delta}\|x_1-x_2\| \leq \frac{M}{\delta} \|(x_1-x_2,d_1-d_2)\| \enspace. \label{eq:lipschitz-case-b}
\end{align}
With the bounds \eqref{eq:lipschitz-case-a} and \eqref{eq:lipschitz-case-b} in Case (a) and (b), respectively, let $L:=\max\{\hat{L},M/\delta\}$ to obtain
\begin{equation}\nonumber
\|g_p(x_1,d_1)- g_p(x_2,d_2)\| \leq L \|(x_1-x_2,d_1-d_2)\| \enspace,
\end{equation}
for any $x_1,x_2\in\mathcal{B}_p$ and $d_1,d_2\in B_\delta(0)$.
\end{proof}

\bibliographystyle{IEEEtran}
\bibliography{bib_lib}

\end{document}